\documentclass[]{ant_tech_report}

\usepackage{wrapfig}
\usepackage{float}

\titleformat*{\paragraph}{\bfseries\itshape}

\makeatletter

\let\c@proposition\relax
\makeatother
\newtheorem{proposition}{Proposition}
\newtheorem*{proposition*}{Proposition}

\title{Rethinking Cross-Layer Information Routing in Diffusion Transformers}

\author{%
\parbox{\textwidth}{\centering
Chao Xu$^{*,3}$,
Maohua Li$^{*,1,2,3,\sharp}$,
Qirui Li$^{3,4,\sharp}$,
Yixuan Xu$^{3}$,
Yanke Zhou$^{1,3,\sharp}$,
Yunhe Li$^{3,5,\sharp}$\\[1mm]
Cuifeng Shen$^{3}$,
Hanlin Tang$^{\dagger,3}$,
Kan Liu$^{3}$,
Tao Lan$^{3}$,
Lin Qu$^{3}$,
Shao-Qun Zhang$^{1,2,\S}$
}}

\affiliation{%
\parbox{\textwidth}{\centering\fontfamily{cmss}\mdseries\selectfont\small
$^1$National Key Laboratory for Novel Software Technology, Nanjing University, China\\
$^2$School of Intelligent Science and Technology, Nanjing University, China\\
$^3$RTP, Alibaba, China\\
$^4$College of Computer Science and Technology, Zhejiang University, China\\
$^5$Department of Computer Science, City University of Hong Kong, Hong Kong
}}
\renewcommand{\contributionlist}{\vspace{-4mm}}

\newcommand{\authorfootnotes}{%
  \begingroup
  \renewcommand{\thefootnote}{\fnsymbol{footnote}}%
  \footnotetext[1]{Equal contribution. Chao Xu initiated the project and built the codebase and infra; Maohua Li refined the idea, led the paper writing and REPA-related work. \quad \textsuperscript{$\dagger$}Project lead \quad $^\sharp$Work done during internship at Alibaba \quad \textsuperscript{\S}Corresponding author}%
  \endgroup
}
\checkdata[E-mail]{\{limaohua, zhangsq\}@lamda.nju.edu.cn, \{shuaichao.xc, wanghaisheng.whs\}@alibaba-inc.com}

\abstract{
Diffusion Transformers (DiTs) have become a de facto backbone of modern visual generation, and nearly every major axis of their design --- tokenization, attention, conditioning, objectives, and latent autoencoders --- has been extensively revisited. The residual stream that governs how information accumulates across layers, however, has been directly inherited from the original Transformer.
In this paper, we present a systematic empirical analysis of cross-layer information flow in DiTs, jointly along depth and denoising timestep, and identify three concrete symptoms of traditional residual addition, namely monotonic forward magnitude inflation, sharp backward gradient decay, and pronounced block-wise redundancy. Motivated by this diagnosis, we propose Diffusion-Adaptive Routing (\textsc{DAR}), a drop-in residual replacement that performs \emph{learnable, timestep-adaptive, and non-incremental} aggregation over the history of sublayer outputs. Moreover, the proposed \textsc{DAR} is compatible with many modern Transformer enhancement methods, such as REPA. On ImageNet $256\times256$, \textsc{DAR} improves SiT-XL/2 by $2.11$ FID ($7.56$ vs.\ $9.67$) and matches the baseline's converged quality with $8.75\times$ fewer training iterations. Stacked on top of REPA, it yields a $2\times$ training acceleration in the early stage, suggesting cross-layer information routing as an underexplored design axis in diffusion modeling, one that operates orthogonally to existing representation-alignment objectives. Beyond pretraining, \textsc{DAR} can also be applied during the fine-tuning stage of large-scale T2I models and preserves high-frequency details during Distribution Matching Distillation. Code is available \href{https://github.com/cg233666/attention-residuals-on-dit}{here}.
}

\begin{document}

\begingroup
\renewcommand{\abstractinfont}{\fontsize{10}{11}\selectfont}
\preto\abstractlist{\begingroup\abstractinfont}
\appto\abstractlist{\looseness=-1\par\endgroup}
\maketitle
\endgroup
\authorfootnotes

\vspace{-1.7em}
\begin{center}
\captionsetup{type=figure}
\centering

\begin{minipage}{0.75\textwidth}
\centering

\begin{subfigure}[t]{0.48\linewidth}
\vspace{13pt}
\centering
\hspace*{5pt}%
\includegraphics[width=0.98\linewidth]{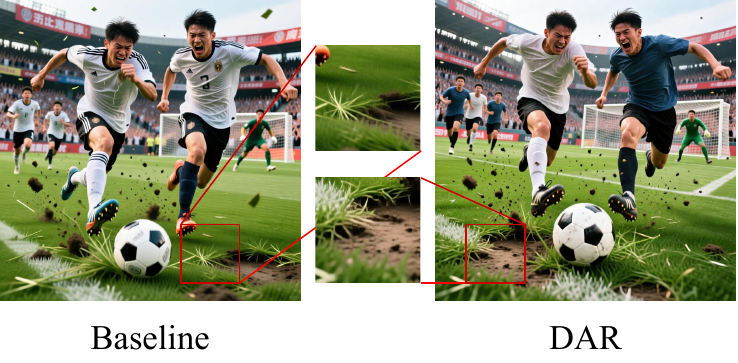}
\caption{DMD visual comparison}
\label{fig:visual_comparison}
\end{subfigure}
\hfill
\begin{subfigure}[t]{0.48\linewidth}
\vspace{0pt}
\centering
\includegraphics[width=\linewidth]{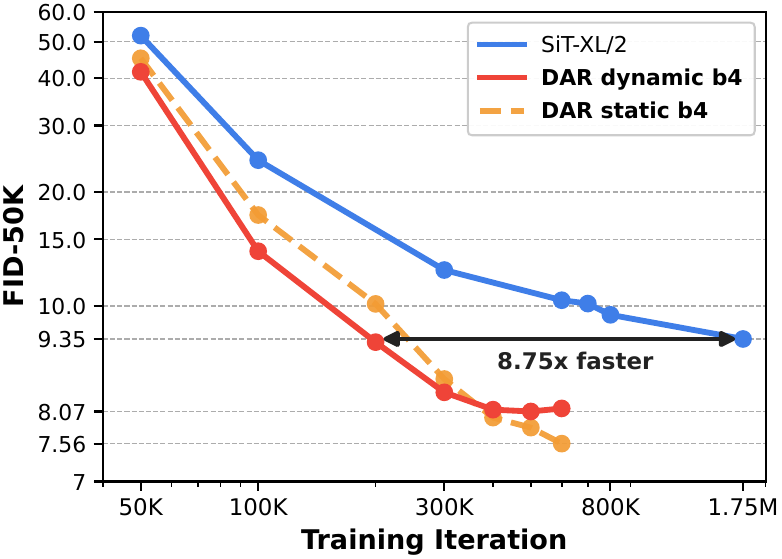}
\caption{Training FID curves}
\label{fig:fid_curve}
\end{subfigure}

\end{minipage}

\caption{Overview of the main empirical results. (a) Our method preserves high-frequency details during DMD. (b) Training FID curves on ImageNet $256\times256$.}
\label{fig:fid_training_curve}
\end{center}

\vspace{-0.5em}
\section{Introduction}  \label{sec:introduction}
Advances in the design and optimization of Diffusion Transformers (DiTs) that replace convolutional U-Nets with token-based Transformer denoisers~\citep{peebles2023scalable} have led to significant breakthroughs in modern visual generation tasks~\citep{wu2025qwen,kong2024hunyuanvideo,flux2024,hacohen2026ltx,cai2025z,seedream2025seedream}. A central challenge for modern visual generation with DiTs is to capture the time-varying dynamics of the denoising process by developing architectural innovations. Recent years have seen extensive efforts devoted to key components of DiTs, including macro structure design~\citep{bao2023all,peebles2023scalable,esser2024scaling,li2024hunyuandit}, attention mechanisms~\citep{xie2024sana,chen2023pixart,peebles2023scalable}, conditioning mechanisms~\citep{tan2025ominicontrol,zhang2025easycontrol}, learning objectives~\citep{yu2025repa,leng2025repae}, latent autoencoders~\citep{yao2025reconstruction,chen2024deep,zheng2025diffusion}, and causal and autoregressive DiTs~\citep{deng2024causal,huang2025self,cheng2025playing}. However, the pre-normalized residual stream in DiTs and its variants --- a fundamental design inherited from standard NLP practice --- has remained largely unchanged, leaving its role in cross-layer information accumulation during time-varying denoising unresolved.

This work starts with an in-depth investigation of cross-layer information routing in DiTs, jointly along depth and denoising timestep. On the one hand, our analysis suggests that this seemingly innocuous default residual addition in DiTs gives rise to three symptoms that emerge in lockstep with depth: hidden-state magnitudes inflate monotonically, backward gradients decay sharply, and adjacent transformer blocks become increasingly redundant, as shown in Fig.~\ref{fig:diagnostic}. Strikingly, these symptoms collectively echo the \emph{PreNorm dilution} phenomenon~\citep{xiong2020layer} recently characterized in Large Language Models (LLMs)~\citep{team2026attention,li2026siamesenorm}. On the other hand, cross-layer information flow within DiTs is inherently time-varying: as denoising progresses across a continuum of noise levels, the intermediate representations that matter most should shift from coarse-structure features in high-noise regimes to fine-detail features in low-noise regimes~\citep{ho2020denoising,sclocchi2025phase}. Thus, the \emph{fixed, time-agnostic, and uniform-weighted} aggregation, as in conventional LLMs, is poorly suited to DiTs.

\begin{figure*}[t]
\centering
\includegraphics[width=\textwidth]{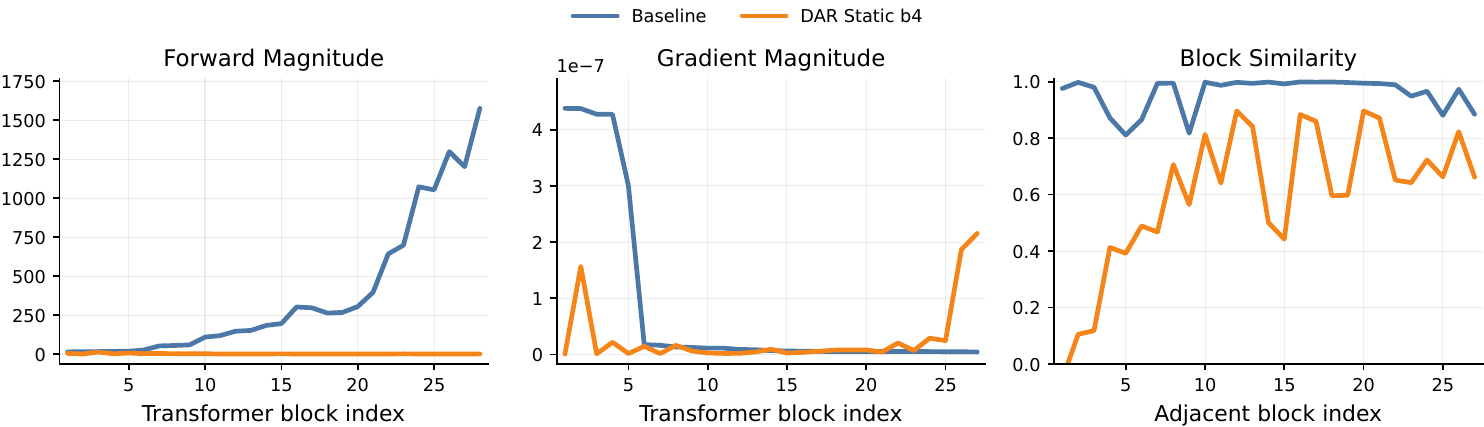}
\vspace{-0.8em}
\caption{Three diagnostic symptoms of standard residual routing in DiTs across depth: Forward magnitude inflation, backward gradient decay, and block-wise redundancy. Measured at $t{=}1.0$.}
\label{fig:diagnostic}
\vspace{-1em}
\end{figure*}

Several works have revisited the depth-wise structure of DiTs. A representative line of research~\citep{bao2023all,tian2024u,chen2025towards,li2024hunyuandit} grafts U-Net-style long skip connections onto DiTs to bridge shallow and deep layers, with the goal of restoring the fixed hierarchical inductive bias of U-Nets. Other studies explore layer- and timestep-dependent feature reuse to accelerate DiT inference~\citep{ma2024learningtocache,chen2024deltadit}. Our key insight is that the denoising timestep --- the very dimension that distinguishes DiTs from a standard Transformer --- should play a vital role in adaptive routing. This motivates depth-wise aggregation mechanisms in DiTs to be \emph{learnable, timestep-adaptive, and non-incremental}, to capture time-varying dynamics.

Building on the above insights, this work elevates cross-layer information routing in DiTs from an inherited convention to an explicit design axis, with contributions on two complementary fronts.

\textbf{On the diagnostic side}, we systematically analyze cross-layer information flow in DiTs across depth and denoising timesteps. We reveal that the three symptoms identified above and illustrated in Fig.~\ref{fig:diagnostic} persist throughout training and vary systematically with the noise level, thereby suggesting that the role of the pre-normalized residual stream extends beyond stabilizing deep training, and exposing a spatiotemporal structure of PreNorm dilution that is invisible to LLM-side analyses.

\textbf{On the methodological side}, we propose Diffusion-Adaptive Routing (\textsc{DAR}), a drop-in residual replacement that performs \emph{learnable, timestep-adaptive, and non-incremental} aggregation. Inspired by~\citet{team2026attention}, we replace the running residual at each sublayer with a softmax attention over preceding sublayer outputs, where the query is computed from the current adaLN-modulated hidden state, allowing the routing mechanism to inherit both content and timestep dependence from DiT's existing conditioning pathway. This preserves the isotropic and homogeneous Transformer stack without introducing manually specified layer pairing, and remains compatible with modern Transformer enhancement methods, such as REPA~\citep{yu2025repa}.

Empirically, on ImageNet 256$\times$256, \textsc{DAR} consistently outperforms vanilla SiT in our experiments, achieving $7.56$ FID with SiT-XL/2 ($2.11\downarrow$ over the baseline) while matching the baseline's converged quality in roughly $8.75 \times$ fewer training iterations. Critically, the gains of \textsc{DAR} are orthogonal to representation-alignment objectives: Combining \textsc{DAR} with REPA~\citep{yu2025repa} yields a $2\times$ training acceleration in the early stage over REPA alone. This suggests that cross-layer information routing is a promising and underexplored direction for improving diffusion models, complementary to existing learning objectives. Quantitatively, the three dilution symptoms identified by our diagnosis tighten in lockstep with these FID gains, linking our diagnosis to the observed gains.

Overall, the main contributions of this paper are summarized as follows
\begin{itemize}
    \item We analyze cross-layer information flow in DiTs across depth and denoising timesteps and identify three concrete symptoms of standard residual routing: \emph{forward magnitude inflation}, \emph{backward gradient decay}, and \emph{block-wise redundancy}.
    \item We propose \textsc{DAR}, a drop-in residual replacement for DiTs that performs \emph{learnable, timestep-adaptive, and non-incremental} aggregation. The design operates purely along the depth dimension, preserving the isotropic and homogeneous Transformer stack, and remains compatible with many modern Transformer enhancement methods, such as REPA.
    \item Our method improves both convergence speed and final quality of diffusion transformers: on SiT, we achieve $8.75 \times$ faster training and a $2.11$ FID improvement over the baseline. Stacked on top of REPA~\citep{yu2025repa}, it yields a $2\times$ training acceleration in the early stage over REPA alone, demonstrating that depth-wise routing operates synergistically with existing representation-alignment objectives.
\end{itemize}
The rest of this paper is organized as follows. Section~\ref{sec:rw} reviews previous studies related to this work. Section~\ref{sec:diagnose} presents an in-depth investigation of cross-layer information flow in DiTs. Section~\ref{sec:dar} introduces DAR. Section~\ref{sec:experiment} conducts experiments to demonstrate the effectiveness of our proposed DAR. Section~\ref{sec:conclusions} concludes this work.

\section{Related Work}  \label{sec:rw}

This section reviews seminal studies on cross-layer information routing and DiT architectures. An extended discussion is provided in Appendix~\ref{app:relatedworks}.
\vspace{-1em}
\paragraph{Evolution of Cross-Layer Information Routing.} Cross-layer information routing in deep networks begins with standard residual connections, where layers communicate through fixed additive recursion~\citep{he2016deep,srivastava2015highway}. Subsequent work mainly improves this residual pathway for optimization stability, including gated or scaled variants such as ReZero~\citep{bachlechner2021rezero}, LayerScale~\citep{touvron2021going}, and DeepNorm~\citep{wang2024deepnet}, which adjust residual strength without fundamentally changing the routing topology. Beyond single-stream propagation, Hyper-Connections~\citep{zhu2024hyper} introduces multi-stream recurrence with learned mixing, which mHC~\citep{xie2025mhc} subsequently refines by imposing doubly stochastic constraints on the mixing for more stable signal propagation at scale. In parallel, another line of work grants layers more direct access to earlier representations, from dense connectivity in DenseNet~\citep{huang2017densely} to learned depth aggregation in DenseFormer~\citep{pagliardini2024denseformer} and explicit depth-wise softmax attention in Attention Residuals~\citep{team2026attention}. Overall, prior studies show a clear transition from fixed residual recursion toward learned, selective, and increasingly dynamic routing across depth. Despite the rapid architectural evolution of generative Transformers, the depth-wise routing dimension remains far less explored than these architectural developments.

\paragraph{Evolution of Diffusion Transformers.} DiTs have evolved from ViT-style U-Net replacements to specialized architectures for scalable generation. U-ViT shows that noisy image patches, timesteps, and conditions can be treated as tokens in a Transformer denoiser while retaining long skip connections~\citep{bao2023all}. DiTs further simplify this design into a pure latent-space Transformer and establish clear scaling behavior~\citep{peebles2023scalable}. Subsequent work has mainly progressed along two directions. One improves multimodal fusion and conditioning. For example, PixArt~\citep{chen2024pixart, chen2023pixart, chen2024pixartdelta} retains conventional cross-attention, whereas MM-DiT~\citep{esser2024scaling} shifts to a unified self-attention framework. This trend also accompanies the adoption of stronger language models as condition encoders: Lumina-T2X~\citep{gao2024lumina}, Playground v3~\citep{Liu2024PlaygroundVI}, and Sana~\citep{xie2024sana} use decoder-only LLMs as text encoders, while Qwen-Image~\citep{wu2025qwen} further extends this design with a vision-language encoder. The other direction advances generative formulations and training objectives. SiT~\citep{ma2024sit} unifies diffusion- and flow-based objectives, while Stable Diffusion 3~\citep{esser2024scaling} stresses rectified-flow training at scale. Notably, REPA~\citep{yu2025repa} accelerates DiT training by introducing a representation-alignment objective that aligns hidden states of DiTs with pretrained visual representations. Overall, the recent evolution of DiTs has focused heavily on backbone scaling, conditioning pathways, and training objectives, whereas the residual pathway itself has remained largely unchanged.



\section{Diagnosing Cross-Layer Information Flow in DiTs}   \label{sec:diagnose}

In this section, we provide an empirical investigation of cross-layer information routing in DiTs, jointly along depth and denoising timestep. 
We analyze two models, that is, a vanilla SiT-XL/2 baseline and a static variant of \textsc{DAR} with chunk size $S = 4$. Both models are checkpointed after $600\text{K}$ training iterations, and diagnostics are computed on $4096$ ImageNet samples. For each transformer block $k \in \{1, \ldots, 28\}$, we record three statistics of its output hidden state $z_k$. The first is the forward magnitude $\mathrm{RMS}(z_k)$ (root-mean-square of the feature values, averaged over batch and tokens). The second is the backward gradient magnitude $\mathrm{RMS}(\partial\mathcal{L}/\partial z_k)$, where $\mathcal{L}$ is the velocity-prediction MSE used for SiT training. The third is the block similarity $\cos(z_k, z_{k+1})$, defined as the per-token cosine similarity between consecutive block outputs averaged over batch and tokens. For \textsc{DAR}, we use $z_k$ to denote the aggregated state passed to block $k+1$; when $k = 28$, $z_k$ denotes the final aggregated state fed to the prediction head.

\begin{figure*}[t]
\centering
\includegraphics[width=\textwidth]{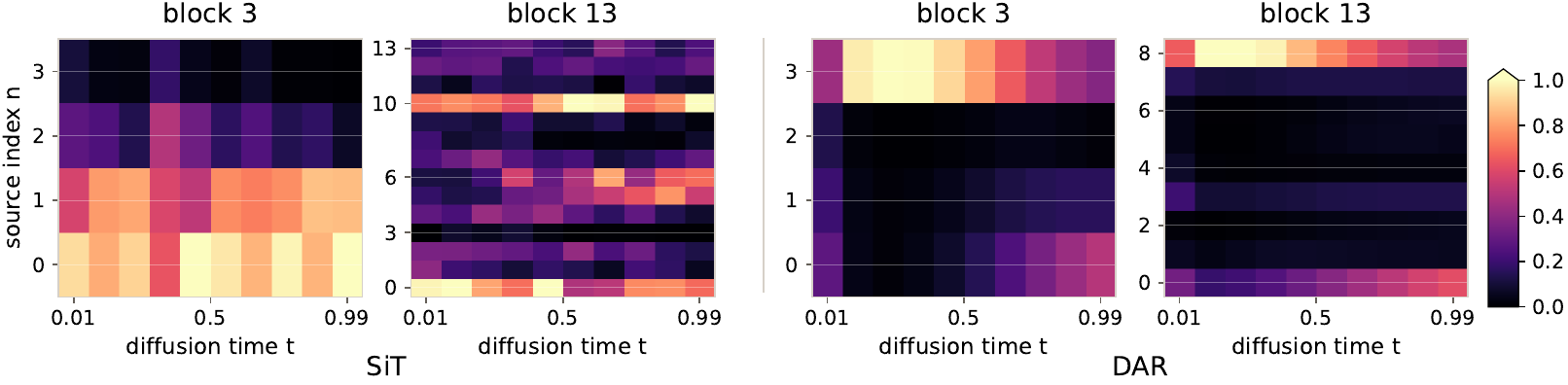}
\vspace{-0.8em}
\caption{Source-mixing patterns across denoising timesteps. For the SiT model with standard residual routing, we add measurement-only scalar gates initialized to $1$ on historical residual sources, keep the forward pass unchanged, and plot loss gradients with respect to the gates as counterfactual source importance; for \textsc{DAR}, we plot the learned softmax routing weights.}
\label{fig:baseline_source_mixing}
\vspace{-1.2em}
\end{figure*}

Fig.~\ref{fig:diagnostic} plots the forward hidden-state magnitude, backward gradient magnitude, and block-wise similarity as functions of the Transformer block index. The blue curves, corresponding to the standard residual baseline, reveal three diagnostic symptoms that all intensify with depth. The forward hidden-state magnitude grows monotonically from $\sim 15.5$ at block 1 to $\sim 1576$ at block 28, corresponding to roughly $100\times$ inflation. Combined with the unit-RMS normalization applied at each block input, this growth forces deeper blocks to produce ever-larger raw outputs in order to retain influence over the residual stream, echoing the \emph{PreNorm dilution} phenomenon characterized in LLMs~\citep{xiong2020layer,team2026attention,li2026siamesenorm}. The backward gradient magnitude drops sharply after the first five blocks. Early blocks receive substantial signal ($\sim 5 \times 10^{-7}$), whereas later blocks are lower by more than an order of magnitude and remain close to zero throughout the deep stack. This pattern suggests that the standard residual pathway provides limited control over gradient flow, leaving deeper layers with substantially weaker optimization signals. The per-token cosine similarity between consecutive block outputs stays above $0.9$ throughout the deep stack, indicating that neighboring deep blocks produce highly similar representations. This high similarity suggests substantial representational redundancy under the standard residual routing.

We next probe the timestep dimension, the key axis that distinguishes DiTs from standard Transformers. For \textsc{DAR}, the softmax weights $\alpha_{i\to l}(t)$ are directly observable; for the SiT baseline, which exposes no router by construction, we attach a scalar gate initialized to $1$ on each historical residual source and read out the gradient of the denoising loss with respect to that gate as a counterfactual importance of how a baseline-equivalent router would reweight each source if one existed, while keeping the forward pass numerically identical to the unmodified baseline. Fig.~\ref{fig:baseline_source_mixing} visualizes both quantities at a shallow and a deep location, and two observations stand out. Although the baseline never sees a router during training, its counterfactual importance map already varies systematically along $t$ at both depths, with the preferred sources at high noise differing visibly from those at low noise --- the standard residual stack exhibits timestep-dependent source preferences, suggesting the value of timestep-conditioned aggregation. \textsc{DAR}'s learned weights provide the missing degree of freedom suggested by this diagnostic: the softmax concentrates sharply on a small subset of historical sources, and this selection itself shifts smoothly with $t$ at both shallow and deep blocks, confirming that timestep-adaptive cross-layer routing is not an externally imposed inductive bias but a latent need of the DiT residual pathway that \textsc{DAR} directly meets.

Taken together, these findings point to an inherent rigidity in standard residual routing, which is associated with three issues: PreNorm dilution driven by residual-stream magnitude growth~\citep{nguyen2019transformers,xiong2020layer,li2026siamesenorm,team2026attention}, imbalanced gradient propagation across depth~\citep{team2026attention,xie2025mhc,zhu2024hyper}, and high feature similarity and redundancy~\citep{jiang2024tracing,song2024sleb,men2025shortgpt,chen2026sortblock}. These observations suggest that standard residuals provide cross-layer propagation, but lack adaptive control over which previous representations should be emphasized or suppressed.

\section{Exploring Cross-Layer Interaction Spaces in DiTs}   \label{sec:dar}

\subsection{Cross-layer Routing in DiTs}

\begin{wrapfigure}{r}{0.43\textwidth}
\vspace{-0.6em}
\centering
\includegraphics[width=0.41\textwidth]{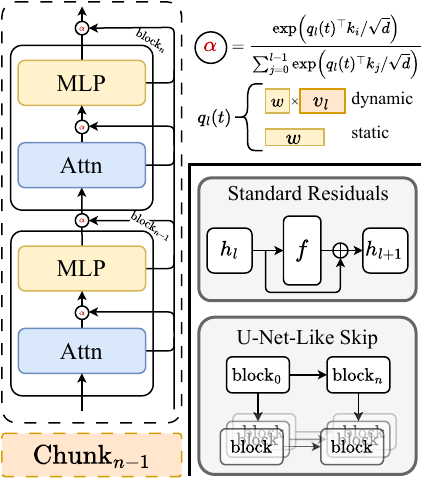}
\caption{Diffusion-Adaptive Routing (\textsc{DAR}) compared with previous methods.}
\label{fig:dar}
\vspace{-0.8em}
\end{wrapfigure}

Motivated by the diagnostic results, we revisit how existing DiT architectures route information. Rather than viewing cross-layer information routing as a post-hoc architectural add-on, we treat it as a fundamental design dimension already implicit in DiTs.

\paragraph{Standard residual routing in DiTs.}
Standard DiTs inherit the residual routing of the original Transformer. For clarity, we treat each self-attention or MLP sublayer as an individual transformation, indexed by $l=1,\ldots,L$
\begin{equation}
    h_{l+1} = h_l + f_l(h_l; t) \ ,
\end{equation}
where $h_l \in \mathbb{R}^{T \times d}$ denotes the hidden token sequence entering sublayer $l$, $t$ is the diffusion or flow timestep, and $f_l$ is the corresponding attention or MLP transformation. Let $v_0=h_1$ be the input embedding and $v_l=f_l(h_l;t)$. Omitting the conditioning signal for simplicity, we obtain
\begin{equation}
    h_l = v_0 + \sum_{i=1}^{l-1} f_i(h_i; t) = \sum_{i=0}^{l-1} v_i \ .
\end{equation}
Standard DiTs already perform a form of cross-layer information routing. However, this routing pattern is fixed, since all previous outputs enter the residual stream with unit coefficients. Thus, standard DiTs cannot explicitly decide which earlier representations should be retrieved or suppressed at a given depth or denoising stage.

\paragraph{U-Net-like skip routing.}
Previous works~\citep{bao2023all,tian2024u,chen2025towards,li2024hunyuandit} introduce a U-Net-like routing pattern for diffusion models. Abstractly, for a deep layer $l$, U-Net-like long skip routing augments its input with a paired shallow representation
\begin{equation}
    \tilde{h}_l = \psi_l\left(h_{l}, h_{\pi(l)}\right) \ ,
\end{equation}
where $\pi(l) < l$ indexes the corresponding shallow layer and $\psi_l$ denotes the skip-fusion operation. The layer update can then be written as
\begin{equation}
    h_{l+1} = \tilde{h}_l + f_l(\tilde{h}_l; t) \ .
\end{equation}
From a routing perspective, U-Net-like skip routing shows that diffusion Transformers can benefit from multi-level feature fusion. Nevertheless, U-Net-like skip routing uses a manually specified topology that weakens the homogeneity underlying Transformer scalability.

\subsection{Diffusion-Adaptive Routing}
\label{sec:DAR}

Drawing on the recently proposed Attention Residuals (AttnRes) framework~\citep{team2026attention}, which replaces fixed residual accumulation with softmax attention over depth, we instantiate Diffusion-Adaptive Routing (\textsc{DAR}) for DiTs with several design choices tailored to the diffusion setting. Let $v_0$ denote the input embedding and $v_i = f_i(h_i; t)$ the output of sublayer $i$, for $i=1,\ldots,L$. Standard residual routing accumulates these sources into a running stream $h_l = \sum_{i=0}^{l-1} v_i$ with unit weights; \textsc{DAR} instead uses softmax-weighted aggregation
\begin{equation} 
    h_l = \sum_{i=0}^{l-1} \alpha_{i \to l}(t)\, v_i 
    \quad\text{with}\quad
    \alpha_{i \to l}(t) = \frac{\exp\bigl(q_l(t)^{\top} k_i / \sqrt{d}\bigr)}{\sum_{j=0}^{l-1} \exp\bigl(q_l(t)^{\top} k_j / \sqrt{d}\bigr)} \ ,
\end{equation}
where $k_i = \mathrm{RMSNorm}(v_i)$ is the key associated with source $v_i$, and the softmax is computed over the source set $\mathcal{S}_l = \{v_0, v_1, \ldots, v_{l-1}\}$. The aggregated $h_l$ then enters the sublayer transformation following $v_l = f_l(h_l; t)$.

\paragraph{Query parameterization.} The per-layer query $q_l(t)$ admits two natural choices
\begin{equation}
    q_l(t) = \begin{cases}
        w_l \ , & \text{(static)} \\[2pt]
        W_q^{(l)}\, v_{l-1} \ , & \text{(dynamic)}
    \end{cases}
\end{equation}
where $w_l \in \mathbb{R}^d$ is a layer-specific learnable vector and $W_q^{(l)} \in \mathbb{R}^{d \times d}$ is a layer-specific projection. Notably, this is a sharp departure from the LLM-side observation in AttnRes, where the dynamic variant improves only marginally over the static one. We attribute this departure to the diffusion timestep dimension unique to DiTs, which is a structural feature absent in the LLM setting and fundamentally reshapes how the per-layer query should be conditioned. We discuss this below.

\paragraph{Timestep injection.} Concretely, the main difference between static and dynamic query parameterization lies in how $t$ enters the per-layer query. The former keeps $w_l$ time-independent by construction, whereas the latter injects $t$ implicitly since the network input $x_t$ is itself a noised latent and further amplified at each sublayer through DiT's adaLN-Zero conditioning pathway. Additionally, we consider an explicit injection variant that augments $w_l$ with the timestep embedding $e(t)$ reused from DiT's existing $t$-embedder, i.e., $q_l(t) = w_l + e(t)$ at no additional parameter cost. The final layer of the $t$-embedder is zero-initialized, so that $e(t)=0$ at initialization and the model exactly recovers the pure static variant at the start of training. Overall, this yields three query variants of timestep injection: pure static, explicit timestep injection, and dynamic. A more detailed comparison that disentangles timestep awareness from input dependence is provided in Section~\ref{sec:time}.

\paragraph{Chunked aggregation.} Retaining all sublayer outputs increases the activation footprint linearly with depth. To reduce this cost, we support a chunked variant that partitions the $L$ sublayers into $N$ chunks of size $S = L/N$. Chunk $n\in\{1,\ldots,N\}$ contains sublayers $(n-1)S+1,\ldots,nS$ and is summarized by $c_n := v_{nS}$, the output of its last sublayer. For any sublayer $l$ in chunk $n$, the source set is replaced by
\begin{equation}
\mathcal{S}_l \;=\; \underbrace{\{c_0, c_1, \ldots, c_{n-1}\}}_{\text{embedding and prior summaries}} \;\cup\; \underbrace{\{v_{(n-1)S+1}, \ldots, v_{l-1}\}}_{\text{current intra-chunk sources}} \ ,
\end{equation}
which consists of the input embedding, summaries from previous chunks, and all preceding sublayer outputs within the current chunk. The softmax aggregation therefore runs over $|\mathcal{S}_l| \le S + N$ sources, reducing source memory from $O(Ld)$ to $O((S + N)d)$. Our chunked design differs from the LLM-side instantiation of AttnRes in both the choice of chunk summary and the chunk size. We analyze this gap and its dependence on depth in Appendix~\ref{app:final_aggregator} and Section~\ref{sec:chunk}.

\section{Experiments}   \label{sec:experiment}

\subsection{Setup}

\paragraph{Implementation details.} Unless otherwise specified, we follow the setup of SiT~\citep{ma2024sit}. We train on ImageNet-1K~\citep{russakovsky2015imagenet} at $256 \times 256$ resolution, with the same data preprocessing as SiT. To ensure a fair comparison, we retrain SiT-XL/2 from scratch under the identical training recipe. For experiments combining DAR with REPA~\citep{yu2025repa}, we follow the original REPA configuration. We do not use any additional training tricks beyond those in the original SiT recipe. Full hyperparameters and experimental details are provided in Appendix~\ref{impl}.

\paragraph{Evaluation.} We report Fréchet Inception Distance (FID;~\citep{heusel2017gans}), Inception Score (IS;~\citep{salimans2016improved}), spatial Fréchet Inception Distance (sFID;~\citep{nash2021generating}), precision and recall (Pre. and Rec.~\citep{kynkaanniemi2019improved}), computed using 50,000 samples. We use both ODE and SDE samplers as in SiT, with 250 function evaluations by default; unless otherwise specified, we report ODE results.

\newcommand{\best}[1]{\textbf{#1}}
\newcommand{\second}[1]{\underline{#1}}

\begin{table*}[t]
\centering
\setlength{\tabcolsep}{3.0pt}
\renewcommand{\arraystretch}{1.12}
\small
\resizebox{\textwidth}{!}{
\begin{tabular}{lcc|ccccc|ccccc}
\toprule
\multirow{2}{*}{\textbf{Method}}
& \multirow{2}{*}{\textbf{Iters.}}
& \multirow{2}{*}{\textbf{Params}}
& \multicolumn{5}{c|}{\textbf{w/o guidance}}
& \multicolumn{5}{c}{\textbf{w/ guidance}} \\
\cmidrule(lr){4-8} \cmidrule(lr){9-13}
& & 
& FID$\downarrow$ & sFID$\downarrow$ & IS$\uparrow$ & Prec.$\uparrow$ & Rec.$\uparrow$
& FID$\downarrow$ & sFID$\downarrow$ & IS$\uparrow$ & Prec.$\uparrow$ & Rec.$\uparrow$ \\
\midrule

\multicolumn{13}{l}{\textit{\textbf{Standard Residuals}}} \\
\midrule

DiT$_{\textsc{ode}}$
& \multirow{2}{*}{1.75M}
& \multirow{2}{*}{675M}
& 10.58 & 5.64 & 114.2 & 0.65 & 0.67
& 2.25 & 4.56 & 239.2 & 0.80 & 0.59 \\

DiT$_{\textsc{sde}}$
& 
& 
& 9.62 & -- & 121.5 & 0.67 & 0.67
& 2.27 & -- & 278.2 & 0.83 & 0.57 \\

SiT$_{\textsc{ode}}$
& \multirow{2}{*}{1.75M}
& \multirow{2}{*}{675M}
& 9.67 & 6.40 & 124.1 & 0.66 & 0.68
& 2.15 & 4.60 & 258.1 & 0.81 & 0.60 \\

SiT$_{\textsc{sde}}$
& 
& 
& 8.61 & 6.32 & 131.7 & 0.68 & 0.67
& 2.06 & 4.49 & 270.3 & 0.82 & 0.59 \\

SiT-Plus$_{\textsc{ode}}$
& \multirow{2}{*}{1M}
& \multirow{2}{*}{752M}
& 10.85 & 5.57 & 115.2 & 0.66 & 0.67
& 2.36 & 4.40 & 244.6 & 0.80 & 0.58 \\

SiT-Plus$_{\textsc{sde}}$
& 
& 
& 10.02 & 5.18 & 119.9 & 0.68 & 0.66
& 2.34 & 4.56 & 254.1 & 0.82 & 0.57 \\

\midrule
\multicolumn{13}{l}{\textit{\textbf{U-Net-Like Routing}}} \\
\midrule

U-ViT-H/2$_{\textsc{ode}}$
& 500K
& 585M
& -- & -- & -- & -- & --
& 2.29 & 5.68 & 263.9 & 0.82 & 0.57 \\

U-DiT-L$_{\textsc{sde}}$
& 250K
& 810M
& 7.54 & 5.27 & 135.5 & 0.70 & 0.66
& 3.00 & 4.40 & 286.6 & \textbf{0.86} & 0.52 \\

\midrule
\multicolumn{13}{l}{\textit{\textbf{Our Method}}} \\
\midrule

Static c4$_{\textsc{ode}}$
& \multirow{2}{*}{600K}
& \multirow{2}{*}{675M}
& 7.56 & 5.18 & 131.1 & 0.69 & 0.68
& 2.08 & 4.42 & 272.9 & 0.83 & \textbf{0.61} \\

Static c4$_{\textsc{sde}}$
& 
& 
& \textbf{6.92} & 5.27 & \textbf{138.8} & 0.70 & 0.67
& 2.23 & 4.49 & \textbf{287.0} & 0.84 & 0.57 \\

Dynamic c4$_{\textsc{ode}}$
& \multirow{2}{*}{500K}
& \multirow{2}{*}{751M}
& 8.07 & \textbf{5.07} & 129.0 & 0.68 & \textbf{0.69}
& \textbf{2.05} & \textbf{4.39} & 270.1 & 0.82 & 0.60 \\

Dynamic c4$_{\textsc{sde}}$
& 
& 
& 7.39 & 5.20 & 134.7 & \textbf{0.71} & 0.67
& 2.17 & 4.49 & 284.8 & 0.83 & 0.57 \\

\bottomrule
\end{tabular}
}
\caption{
System-level comparison on ImageNet $256 \times 256$ generation, with and without classifier-free guidance~\citep{ho2022classifier}, with the best results marked in \textbf{bold}. We use CFG with $w=1.5$. Here, c4 denotes a chunk size of $4$.
}
\label{tab:imagenet256_main}
\vspace{-0.8em}
\end{table*}

\subsection{Better Quality and Faster Convergence}
\label{sec:exp_quality_speed}

We first show in Tab.~\ref{tab:imagenet256_main} that \textsc{DAR} improves both final quality and convergence speed. Compared to the vanilla SiT-XL/2 baseline trained for $1.75\text{M}$ iterations, the static variant of \textsc{DAR} achieves a substantially better FID of $6.92$ (SDE) without CFG, while training for only $600\text{K}$ iterations. The dynamic variant attains the best ODE FID with CFG ($2.05$), again outperforming the SiT baseline with far fewer training iterations. To rule out the possibility that these gains arise simply from increased model size, we further train SiT-Plus, a widened SiT-XL/2 matched to the $752$M parameter count of dynamic c4. Despite using $2\times$ the training budget, SiT-Plus still trails \textsc{DAR} by a wide margin, confirming that the gains of \textsc{DAR} cannot be explained simply by parameter scaling.

A natural follow-up question is whether the gains of \textsc{DAR} merely reproduce what can already be achieved by equipping DiTs with U-Net-like skip pathways. We therefore compare against two representative models from this family in Tab.~\ref{tab:imagenet256_main}: U-ViT-H/2~\citep{bao2023all} and U-DiT-L~\citep{tian2024u}. Under SDE with CFG, \textsc{DAR} static c4 outperforms U-DiT-L by $0.77$ FID while using only $0.83\times$ as many parameters. Under ODE, \textsc{DAR} dynamic c4 further improves over U-ViT-H/2 by $0.24$ FID. The gap is informative: hand-designed skip topologies wire shallow and deep layers together at predetermined depths and with fixed, time-invariant fusion weights, whereas \textsc{DAR} replaces these manual choices with learned and timestep-adaptive aggregation. Most importantly, \textsc{DAR} preserves the isotropic, homogeneous Transformer stack that underlies modern scaling.

\subsection{Timestep Awareness Is Crucial for Routing in DiTs}
\label{sec:time}

\begin{table*}[t]
\centering
\setlength{\tabcolsep}{5pt}
\renewcommand{\arraystretch}{1.12}
\small

\begin{minipage}[t]{0.48\textwidth}
\centering
\begin{tabular}{lccc}
\toprule
\textbf{Method} & \textbf{100K} & \textbf{200K} & \textbf{400K} \\
\midrule
Static w/o $t$-injection & 22.36 & 15.47 & 11.51 \\
Dynamic            & 13.95 & 9.29 & 8.10 \\
Static w/ $t$-injection  & 17.39 & 10.12 & 7.97 \\

\bottomrule
\end{tabular}
\caption{Ablation on timestep awareness in \textsc{DAR}. We report FID$\downarrow$ at different training iterations.}
\label{tab:timestep_awareness}
\end{minipage}
\hfill
\begin{minipage}[t]{0.48\textwidth}
\centering
\begin{tabular}{lccc}
\toprule
\textbf{Method} & \textbf{100K} & \textbf{200K} & \textbf{300K} \\
\midrule
SiT + REPA  & 9.89 & 6.89 & 6.29 \\
DAR + REPA  & 7.09 & 5.92 & 5.68 \\
\bottomrule
\end{tabular}
\caption{Compatibility with REPA. We report FID$\downarrow$ at different training iterations.}
\label{tab:repa_orthogonal}
\end{minipage}
\vspace{-1.5em}
\end{table*}

\begin{wrapfigure}{l}{0.45\textwidth}
    \vspace{-0.6em}
    \centering
    \includegraphics[width=0.44\textwidth]{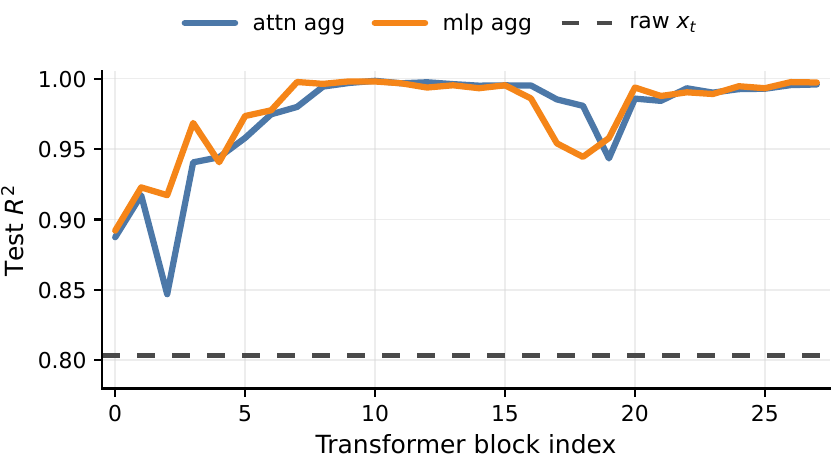}
    \caption{Linear-probe test $R^2$ for decoding the scalar denoising timestep $t$ from the aggregated hidden states that feed the router at each depth in DAR-Dynamic.}
    \label{fig:fig_t}
    \vspace{-0.8em}
    \end{wrapfigure}

Cross-layer information routing in DiTs differs fundamentally from its LLM counterpart in one key respect: the denoising timestep $t$ is a first-class control signal that modulates every layer's computation, and the very motivation for adaptive routing, that different noise levels demand different mixtures of shallow and deep features, presupposes that the router itself is aware of $t$. The three query parameterizations of \S\ref{sec:DAR} differ precisely in how this awareness is introduced. The pure static variant, $q_l = w_l$, leaves the router blind to the timestep. The dynamic variant, $\smash{q_l(t) = W_q^{(l)} v_{l-1}}$, derives its query from the most recent sublayer output, so any timestep information already encoded in $v_{l-1}$ is naturally inherited by the routing weights. We refer to this as implicit timestep injection. The explicit-injection variant, $q_l(t) = w_l + e(t)$, instead reuses DiT's existing timestep embedding to add a direct timestep signal to an otherwise time-invariant query, at no additional parameter cost. As shown in Tab.~\ref{tab:timestep_awareness}, both timestep-aware variants substantially outperform the timestep-blind baseline at matched compute. This supports the view that timestep awareness---regardless of how it is injected---is the key ingredient.

The effectiveness of the dynamic variant rests on a non-trivial premise: the input to the dynamic query, $v_{l-1}$, actually retains sufficient timestep information for $\smash{W_q^{(l)}}$ to extract. We test this premise directly with a linear-probe diagnostic. For a frozen DAR-Dynamic checkpoint, we sweep $t$ over a uniform grid while holding the underlying $(x_0, x_1)$ image-noise pairs fixed, collect the aggregated hidden state $h_l$ that feeds each sublayer's router, and fit a ridge regressor that maps the pooled feature back to the scalar $t$ on disjoint pair-level train/test splits. Fig.~\ref{fig:fig_t} reports the resulting test $R^2$ across all $28$ blocks. Both the attention and MLP aggregator inputs sit well above the raw input latents $x_t$ baseline ($R^2 \approx 0.80$) at every depth, exceed $0.95$ within the first five blocks, and remain close to $1.0$ throughout the deep stack. Thus the timestep is not merely present but linearly decodable from the very tensors the dynamic query consumes, confirming that $\smash{q_l(t) = W_q^{(l)} v_{l-1}}$ enjoys direct access to a strong $t$-signal without any explicit conditioning, and is consistent with the mechanism underlying the early-training advantage of the dynamic variant.

\subsection{DAR Is Orthogonal to REPA}
REPA~\citep{yu2025repa} accelerates DiT training by aligning intermediate hidden states with a pretrained visual encoder, an objective-level intervention that leaves the residual pathway untouched. \textsc{DAR}, in contrast, restructures how those hidden states are aggregated across depth, a purely architectural change that is agnostic to the training loss. The two interventions therefore act along orthogonal axes, and a natural question is whether their gains compound or merely overlap. Tab.~\ref{tab:repa_orthogonal} answers this directly: stacking \textsc{DAR} on top of REPA improves FID from $9.89$ to $7.09$ at $100\text{K}$ iterations and from $6.89$ to $5.92$ at $200\text{K}$ iterations. Notably, \textsc{DAR}+REPA at $100\text{K}$ already matches REPA alone at $200\text{K}$, a roughly $2\times$ early-stage speedup, while \textsc{DAR}+REPA at $200\text{K}$ surpasses even the $300\text{K}$ FID of REPA alone, indicating that the routing-level and representation-level accelerations compound rather than offset each other.

\subsection{Chunk Size Choices}
\label{sec:chunk}

\begin{wraptable}{r}{0.40\textwidth}
\vspace{-1.0em}
\centering
\setlength{\tabcolsep}{6pt}
\renewcommand{\arraystretch}{1.12}
\small
\begin{tabular}{lccc}
\toprule
Chunk size & 1 & 4 & 8 \\
\midrule
FID$\downarrow$ & 10.41 & 8.39  & 11.14 \\
IS$\uparrow$    & 107.2 & 121.7 & 103.51 \\
\bottomrule
\end{tabular}
\caption{Effects of chunk size $S$ in chunked aggregation on SiT-XL/2, ImageNet $256{\times}256$, no classifier-free guidance.}
\label{tab:chunk_size}
\vspace{-0.8em}
\end{wraptable}

The chunked aggregation in \S\ref{sec:DAR} exposes a single knob, the chunk size $S$, that interpolates between two extremes: $S=1$ degenerates to the dense, all-source variant where every sublayer output enters the routing softmax, while large $S$ collapses each chunk to a single summary $c_n$ and aggressively compresses the source set. Tab.~\ref{tab:chunk_size} reports a sweep of $S \in \{1, 4, 8\}$ on \textsc{DAR} under matched compute ($300\text{K}$ iterations) and reveals a clear U-shaped pattern with $S = 4$ at the bottom. We now show that this U-shaped pattern is not coincidental: the cost can be decomposed into two competing terms whose unique minimum lies precisely in the $S \approx 4$ regime predicted for \textsc{DAR}.

Under a mild rate-distortion model, the cost per aggregator decomposes as
\begin{equation}
\mathcal{L}(S) \;=\; \log\bigl(L/S + S\bigr) \;+\; \alpha \, \log S 
\quad\text{with}\quad \alpha \in (0, 1) \ ,
\label{eq:chunk_cost}
\end{equation}
where $L$ is the total number of sublayers and $S \in (0, L]$ denotes the chunk size. The first term is the maximum-entropy lower bound on a softmax over $|\mathcal{S}|$ slots and caps the routing precision attainable by a bounded-norm query; the second is the rate-distortion cost of compressing $S$ sublayer outputs into a single $d$-dimensional summary, with $\alpha$ discounting the loss to reflect partial recoverability through vertical attention over earlier summaries.

\begin{proposition}[U-shaped cost of chunked aggregation] \label{prop:u_shape}
Let $L>0$ and $\alpha\in(0,1)$. Then $\mathcal{L}(S)$ in
Eq.~\eqref{eq:chunk_cost} is strictly decreasing on $(0,S^\star)$ and strictly
increasing on $(S^\star,\infty)$, where
\begin{equation}
S^\star=\sqrt{L \cdot \frac{1-\alpha}{1+\alpha}}.
\label{eq:S_star}
\end{equation}
Consequently, $\mathcal{L}(S)$ is U-shaped and has a unique global minimizer at
$S^\star$.
\end{proposition}

A detailed proof and further analysis are provided in Appendix~\ref{app:chunk_proof}.

\subsection{Large-Scale T2I Model Post-Training}

Beyond ImageNet, we apply Distribution Matching Distillation~\citep{yin2024one,yin2024improved} to Qwen-Image~\citep{wu2025qwen} equipped with \textsc{DAR}. We find that adaptive cross-layer information routing helps the distilled model preserve high-frequency details, including sharp edges and fine textures, which are easily attenuated during aggressive few-step distillation. Full setup and samples are deferred to Appendix~\ref{t2i}.

\section{Conclusion}  \label{sec:conclusions}
In this paper, we conducted a systematic investigation of cross-layer information routing in DiTs, jointly along depth and denoising timestep, identified three symptoms of the pre-normalized residual stream inherited from the original Transformer---forward magnitude inflation, backward gradient decay, and block-wise redundancy---and accordingly proposed \textsc{DAR}, a drop-in residual replacement that performs learnable, timestep-adaptive, and non-incremental aggregation. On ImageNet $256{\times}256$, \textsc{DAR} attains a best FID of $6.92$ on SiT-XL/2, matches the baseline with $8.75\times$ fewer iterations, and further delivers a $2\times$ early-stage speedup over REPA. These results identify cross-layer routing as an underexplored design axis that complements prevailing advances in diffusion modeling. Further discussion of limitations and future work is provided in Appendix~\ref{limitations}.

\section*{Acknowledgments and Disclosure of Funding}
This research was supported by the National Natural Science Foundation of China (62406138, 62632005). Maohua Li also received support from the Postgraduate Research \& Practice Innovation Program of Jiangsu Province (Grant No. 26CXJH0431).

This research was performed during the academic cooperation between LAMDA and RTP-LLM (Alibaba). The proposed method was evaluated and deployed in Alibaba's internal systems and on the Wuli platform.

\bibliographystyle{plainnat}
\bibliography{references}

@inproceedings{peebles2023scalable,
  title={Scalable diffusion models with transformers},
  author={Peebles, William and Xie, Saining},
  booktitle={Proceedings of the IEEE/CVF International Conference on Computer Vision},
  pages={4195--4205},
  year={2023}
}

@inproceedings{bao2023all,
  title={{All are worth words: A ViT backbone for diffusion models}},
  author={Bao, Fan and Nie, Shen and Xue, Kaiwen and Cao, Yue and Li, Chongxuan and Su, Hang and Zhu, Jun},
  booktitle={Proceedings of the IEEE/CVF Conference on Computer Vision and Pattern Recognition},
  pages={22669--22679},
  year={2023}
}

@inproceedings{esser2024scaling,
  title={Scaling rectified flow transformers for high-resolution image synthesis},
  author={Esser, Patrick and Kulal, Sumith and Blattmann, Andreas and Entezari, Rahim and M{\"u}ller, Jonas and Saini, Harry and Levi, Yam and Lorenz, Dominik and Sauer, Axel and Boesel, Frederic and others},
  booktitle={Proceedings of the 41st International Conference on Machine Learning},
  year={2024},
  pages={12606--12633}
}

@inproceedings{ma2024sit,
  title={{SiT: Exploring flow and diffusion-based generative models with scalable interpolant transformers}},
  author={Ma, Nanye and Goldstein, Mark and Albergo, Michael S and Boffi, Nicholas M and Vanden-Eijnden, Eric and Xie, Saining},
  booktitle={Proceedings of the 18th European Conference on Computer Vision},
  pages={23--40},
  year={2024}
}

@article{chen2023pixart,
  title={{PixArt-$\alpha$: Fast training of diffusion transformer for photorealistic text-to-image synthesis}},
  author={Chen, Junsong and Yu, Jincheng and Ge, Chongjian and Yao, Lewei and Xie, Enze and Wu, Yue and Wang, Zhongdao and Kwok, James and Luo, Ping and Lu, Huchuan and others},
  journal={arXiv preprint arXiv:2310.00426},
  year={2023}
}

@inproceedings{chen2024pixart,
  title={{PixArt-$\Sigma$: Weak-to-strong training of diffusion transformer for 4K text-to-image generation}},
  author={Chen, Junsong and Ge, Chongjian and Xie, Enze and Wu, Yue and Yao, Lewei and Ren, Xiaozhe and Wang, Zhongdao and Luo, Ping and Lu, Huchuan and Li, Zhenguo},
  booktitle={Proceedings of the 18th European Conference on Computer Vision},
  pages={74--91},
  year={2024}
}

@article{chen2024pixartdelta,
  title={{PixArt-$\delta$: Fast and controllable image generation with latent consistency models}},
  author={Chen, Junsong and Wu, Yue and Luo, Simian and Xie, Enze and Paul, Sayak and Luo, Ping and Zhao, Hang and Li, Zhenguo},
  journal={arXiv preprint arXiv:2401.05252},
  year={2024}
}

@article{wu2025qwen,
  title={{Qwen-Image technical report}},
  author={Wu, Chenfei and Li, Jiahao and Zhou, Jingren and Lin, Junyang and Gao, Kaiyuan and Yan, Kun and Yin, Sheng-ming and Bai, Shuai and Xu, Xiao and Chen, Yilei and Chen, Yuxiang and Tang, Zecheng and Zhang, Zekai and Wang, Zhengyi and Yang, An and Yu, Bowen and Cheng, Chen and Liu, Dayiheng and Li, Deqing and Zhang, Hang and Meng, Hao and Wei, Hu and Ni, Jingyuan and Chen, Kai and Cao, Kuan and Peng, Liang and Qu, Lin and Wu, Minggang and Wang, Peng and Yu, Shuting and Wen, Tingkun and Feng, Wensen and Xu, Xiaoxiao and Wang, Yi and Zhang, Yichang and Zhu, Yongqiang and Wu, Yujia and Cai, Yuxuan and Liu, Zenan},
  journal={arXiv preprint arXiv:2508.02324},
  year={2025}
}

@article{gao2024lumina,
  title={{Lumina-T2X: Transforming text into any modality, resolution, and duration via flow-based large diffusion transformers}},
  author={Gao, Peng and Zhuo, Le and Liu, Dongyang and Du, Ruoyi and Luo, Xu and Qiu, Longtian and Zhang, Yuhang and Lin, Chen and Huang, Rongjie and Geng, Shijie and others},
  journal={arXiv preprint arXiv:2405.05945},
  year={2024}
}

@article{Liu2024PlaygroundVI,
  title={{Playground v3: Improving text-to-image alignment with deep-fusion large language models}},
  author={Bingchen Liu and Ehsan Akhgari and Alexander Visheratin and Aleks Kamko and Linmiao Xu and Shivam Shrirao and Jo{\~a}o Pedro Gandarela de Souza and Suhail Doshi and Daiqing Li},
  journal={arXiv preprint arXiv:2409.10695},
  year={2024}
}

@article{xie2024sana,
  title={{SANA: Efficient high-resolution image synthesis with linear diffusion transformers}},
  author={Xie, Enze and Chen, Junsong and Chen, Junyu and Cai, Han and Tang, Haotian and Lin, Yujun and Zhang, Zhekai and Li, Muyang and Zhu, Ligeng and Lu, Yao and others},
  journal={arXiv preprint arXiv:2410.10629},
  year={2024}
}

@inproceedings{yu2025repa,
  title={{Representation alignment for generation: Training diffusion transformers is easier than you think}},
  author={Sihyun Yu and Sangkyung Kwak and Huiwon Jang and Jongheon Jeong and Jonathan Huang and Jinwoo Shin and Saining Xie},
  year={2025},
  booktitle={Proceedings of the 13th International Conference on Learning Representations},
  pages={87400--87442}
}

@article{leng2025repae,
  title={{REPA-E: Unlocking VAE for end-to-end tuning with latent diffusion transformers}},
  author={Xingjian Leng and Jaskirat Singh and Yunzhong Hou and Zhenchang Xing and Saining Xie and Liang Zheng},
  year={2025},
  journal={arXiv preprint arXiv:2504.10483}
}

@inproceedings{he2016deep,
  title={Deep residual learning for image recognition},
  author={He, Kaiming and Zhang, Xiangyu and Ren, Shaoqing and Sun, Jian},
  booktitle={Proceedings of the IEEE Conference on Computer Vision and Pattern Recognition},
  pages={770--778},
  year={2016}
}

@article{xie2025mhc,
  title={{mHC: Manifold-constrained hyper-connections}},
  author={Xie, Zhenda and Wei, Yixuan and Cao, Huanqi and Zhao, Chenggang and Deng, Chengqi and Li, Jiashi and Dai, Damai and Gao, Huazuo and Chang, Jiang and Yu, Kuai and Zhao, Liang and Zhou, Shangyan and Xu, Zhean and Zhang, Zhengyan and Zeng, Wangding and Hu, Shengding and Wang, Yuqing and Yuan, Jingyang and Wang, Lean and Liang, Wenfeng},
  journal={arXiv preprint arXiv:2512.24880},
  year={2025}
}

@article{team2026attention,
  title={Attention residuals},
  author={{Kimi Team} and Chen, Guangyu and Zhang, Yu and Su, Jianlin and Xu, Weixin and Pan, Siyuan and Wang, Yaoyu and Wang, Yucheng and Chen, Guanduo and Yin, Bohong and others},
  journal={arXiv preprint arXiv:2603.15031},
  year={2026}
}

@article{srivastava2015highway,
  title={Highway networks},
  author={Srivastava, Rupesh Kumar and Greff, Klaus and Schmidhuber, J{\"u}rgen},
  journal={arXiv preprint arXiv:1505.00387},
  year={2015}
}

@article{zhu2024hyper,
  title={Hyper-connections},
  author={Zhu, Defa and Huang, Hongzhi and Huang, Zihao and Zeng, Yutao and Mao, Yunyao and Wu, Banggu and Min, Qiyang and Zhou, Xun},
  journal={arXiv preprint arXiv:2409.19606},
  year={2024}
}

@inproceedings{pagliardini2024denseformer,
  title={{DenseFormer: Enhancing information flow in transformers via depth weighted averaging}},
  author={Pagliardini, Matteo and Mohtashami, Amirkeivan and Fleuret, Francois and Jaggi, Martin},
  volume={37},
  pages={136479--136508},
  year={2024},
  booktitle={Advances in Neural Information Processing Systems}
}

@inproceedings{bachlechner2021rezero,
  title={{ReZero is all you need: Fast convergence at large depth}},
  author={Bachlechner, Thomas and Majumder, Bodhisattwa Prasad and Mao, Henry and Cottrell, Gary and McAuley, Julian},
  booktitle={Proceedings of the 37th Conference on Uncertainty in Artificial Intelligence},
  pages={1352--1361},
  year={2021}
}

@inproceedings{touvron2021going,
  title={Going deeper with image transformers},
  author={Touvron, Hugo and Cord, Matthieu and Sablayrolles, Alexandre and Synnaeve, Gabriel and J{\'e}gou, Herv{\'e}},
  booktitle={Proceedings of the IEEE/CVF International Conference on Computer Vision},
  pages={32--42},
  year={2021}
}

@article{wang2024deepnet,
  title={{DeepNet: Scaling transformers to 1,000 layers}},
  author={Wang, Hongyu and Ma, Shuming and Dong, Li and Huang, Shaohan and Zhang, Dongdong and Wei, Furu},
  journal={IEEE Transactions on Pattern Analysis and Machine Intelligence},
  volume={46},
  number={10},
  pages={6761--6774},
  year={2024},
  publisher={IEEE}
}

@inproceedings{huang2017densely,
  title={Densely connected convolutional networks},
  author={Huang, Gao and Liu, Zhuang and van der Maaten, Laurens and Weinberger, Kilian Q.},
  booktitle={Proceedings of the IEEE Conference on Computer Vision and Pattern Recognition},
  pages={4700--4708},
  year={2017}
}

@inproceedings{ho2020denoising,
  title={Denoising diffusion probabilistic models},
  author={Ho, Jonathan and Jain, Ajay and Abbeel, Pieter},
  volume={33},
  pages={6840--6851},
  year={2020},
  booktitle={Advances in Neural Information Processing Systems}
}

@article{lipman2022flow,
  title={Flow matching for generative modeling},
  author={Lipman, Yaron and Chen, Ricky TQ and Ben-Hamu, Heli and Nickel, Maximilian and Le, Matt},
  journal={arXiv preprint arXiv:2210.02747},
  year={2022}
}

@article{song2020score,
  title={Score-based generative modeling through stochastic differential equations},
  author={Song, Yang and Sohl-Dickstein, Jascha and Kingma, Diederik P and Kumar, Abhishek and Ermon, Stefano and Poole, Ben},
  journal={arXiv preprint arXiv:2011.13456},
  year={2020}
}

@article{liu2022flow,
  title={Flow straight and fast: Learning to generate and transfer data with rectified flow},
  author={Liu, Xingchao and Gong, Chengyue and Liu, Qiang},
  journal={arXiv preprint arXiv:2209.03003},
  year={2022}
}

@article{li2026siamesenorm,
  title={{SiameseNorm: Breaking the barrier to reconciling Pre/Post-Norm}},
  author={Li, Tianyu and Han, Dongchen and Cao, Zixuan and Huang, Haofeng and Zhou, Mengyu and Chen, Ming and Zhao, Erchao and Jiang, Xiaoxi and Jiang, Guanjun and Huang, Gao},
  journal={arXiv preprint arXiv:2602.08064},
  year={2026}
}

@inproceedings{chen2026sortblock,
  title={{SortBlock: Similarity-aware feature reuse for diffusion model}},
  author={Chen, Hanqi and Zhang, Xu and Guan, Xiaoliu and Jiang, Lielin and Wang, Guanzhong and Chen, Zeyu and Liu, Yi},
  booktitle={Proceedings of the AAAI Conference on Artificial Intelligence},
  pages={2859--2867},
  year={2026}
}

@inproceedings{men2025shortgpt,
  title={{ShortGPT: Layers in large language models are more redundant than you expect}},
  author={Men, Xin and Xu, Mingyu and Zhang, Qingyu and Yuan, Qianhao and Wang, Bingning and Lin, Hongyu and Lu, Yaojie and Han, Xianpei and Chen, Weipeng},
  booktitle={Findings of the Association for Computational Linguistics: ACL 2025},
  pages={20192--20204},
  year={2025}
}

@inproceedings{song2024sleb,
  title={{SLEB: Streamlining LLMs through redundancy verification and elimination of transformer blocks}},
  author={Song, Jiwon and Oh, Kyungseok and Kim, Taesu and Kim, Hyungjun and Kim, Yulhwa and Kim, Jae-Joon},
  booktitle={Proceedings of the 41st International Conference on Machine Learning},
  pages={46136--46155},
  year={2024}
}

@article{jiang2024tracing,
  title={Tracing representation progression: Analyzing and enhancing layer-wise similarity},
  author={Jiang, Jiachen and Zhou, Jinxin and Zhu, Zhihui},
  journal={arXiv preprint arXiv:2406.14479},
  year={2024}
}

@inproceedings{nguyen2019transformers,
  title={Transformers without tears: Improving the normalization of self-attention},
  author={Nguyen, Toan Q and Salazar, Julian},
  booktitle={Proceedings of the 16th International Conference on Spoken Language Translation},
  year={2019}
}

@inproceedings{xiong2020layer,
  title={On layer normalization in the transformer architecture},
  author={Xiong, Ruibin and Yang, Yunchang and He, Di and Zheng, Kai and Zheng, Shuxin and Xing, Chen and Zhang, Huishuai and Lan, Yanyan and Wang, Liwei and Liu, Tieyan},
  booktitle={Proceedings of the 37th International Conference on Machine Learning},
  pages={10524--10533},
  year={2020}
}

@article{kong2024hunyuanvideo,
  title={{HunyuanVideo: A systematic framework for large video generative models}},
  author={Kong, Weijie and Tian, Qi and Zhang, Zijian and Min, Rox and Dai, Zuozhuo and Zhou, Jin and Xiong, Jiangfeng and Li, Xin and Wu, Bo and Zhang, Jianwei and others},
  journal={arXiv preprint arXiv:2412.03603},
  year={2024}
}

@misc{flux2024,
  author={Black Forest Labs},
  title={{FLUX}},
  year={2024},
  howpublished={\url{https://github.com/black-forest-labs/flux}}
}

@article{hacohen2026ltx,
  title={{LTX-2: Efficient joint audio-visual foundation model}},
  author={HaCohen, Yoav and Brazowski, Benny and Chiprut, Nisan and Bitterman, Yaki and Kvochko, Andrew and Berkowitz, Avishai and Shalem, Daniel and Lifschitz, Daphna and Moshe, Dudu and Porat, Eitan and Richardson, Eitan and Shiran, Guy and Chachy, Itay and Chetboun, Jonathan and Finkelson, Michael and Kupchick, Michael and Zabari, Nir and Guetta, Nitzan and Kotler, Noa and Bibi, Ofir and Gordon, Ori and Panet, Poriya and Benita, Roi and Armon, Shahar and Kulikov, Victor and Inger, Yaron and Shiftan, Yonatan and Melumian, Zeev and Farbman, Zeev},
  journal={arXiv preprint arXiv:2601.03233},
  year={2026}
}

@article{cai2025z,
  title={{Z-Image: An efficient image generation foundation model with single-stream diffusion transformer}},
  author={Cai, Huanqia and Cao, Sihan and Du, Ruoyi and Gao, Peng and Hoi, Steven and Hou, Zhaohui and Huang, Shijie and Jiang, Dengyang and Jin, Xin and Li, Liangchen and others},
  journal={arXiv preprint arXiv:2511.22699},
  year={2025}
}

@article{seedream2025seedream,
  title={Seedream 4.0: Toward next-generation multimodal image generation},
  author={{Team Seedream} and Chen, Yunpeng and Gao, Yu and Gong, Lixue and Guo, Meng and Guo, Qiushan and Guo, Zhiyao and Hou, Xiaoxia and Huang, Weilin and Huang, Yixuan and Jian, Xiaowen and Kuang, Huafeng and Lai, Zhichao and Li, Fanshi and Li, Liang and Lian, Xiaochen and Liao, Chao and Liu, Liyang and Liu, Wei and Lu, Yanzuo and Luo, Zhengxiong and Ou, Tongtong and Shi, Guang and Shi, Yichun and Sun, Shiqi and Tian, Yu and Tian, Zhi and Wang, Peng and Wang, Rui and Wang, Xun and Wang, Ye and Wu, Guofeng and Wu, Jie and Wu, Wenxu and Wu, Yonghui and Xia, Xin and Xiao, Xuefeng and Xu, Shuang and Yan, Xin and Yang, Ceyuan and Yang, Jianchao and Zhai, Zhonghua and Zhang, Chenlin and Zhang, Heng and Zhang, Qi and Zhang, Xinyu and Zhang, Yuwei and Zhao, Shijia and Zhao, Wenliang and Zhu, Wenjia},
  journal={arXiv preprint arXiv:2509.20427},
  year={2025}
}

@article{li2024hunyuandit,
  title={{Hunyuan-DiT: A powerful multi-resolution diffusion transformer with fine-grained Chinese understanding}},
  author={Zhimin Li and Jianwei Zhang and Qin Lin and Jiangfeng Xiong and Yanxin Long and Xinchi Deng and Yingfang Zhang and Xingchao Liu and Minbin Huang and Zedong Xiao and Dayou Chen and Jiajun He and Jiahao Li and Wenyue Li and Chen Zhang and Rongwei Quan and Jianxiang Lu and Jiabin Huang and Xiaoyan Yuan and Xiaoxiao Zheng and Yixuan Li and Jihong Zhang and Chao Zhang and Meng Chen and Jie Liu and Zheng Fang and Weiyan Wang and Jinbao Xue and Yangyu Tao and Jianchen Zhu and Kai Liu and Sihuan Lin and Yifu Sun and Yun Li and Dongdong Wang and Mingtao Chen and Zhichao Hu and Xiao Xiao and Yan Chen and Yuhong Liu and Wei Liu and Di Wang and Yong Yang and Jie Jiang and Qinglin Lu},
  year={2024},
  journal={arXiv preprint arXiv:2405.08748}
}

@article{zheng2025diffusion,
  title={Diffusion transformers with representation autoencoders},
  author={Zheng, Boyang and Ma, Nanye and Tong, Shengbang and Xie, Saining},
  journal={arXiv preprint arXiv:2510.11690},
  year={2025}
}

@inproceedings{tan2025ominicontrol,
  title={{OminiControl: Minimal and universal control for diffusion transformer}},
  author={Tan, Zhenxiong and Liu, Songhua and Yang, Xingyi and Xue, Qiaochu and Wang, Xinchao},
  booktitle={Proceedings of the IEEE/CVF International Conference on Computer Vision},
  pages={14940--14950},
  year={2025}
}

@inproceedings{zhang2025easycontrol,
  title={{EasyControl: Adding efficient and flexible control for diffusion transformer}},
  author={Zhang, Yuxuan and Yuan, Yirui and Song, Yiren and Wang, Haofan and Liu, Jiaming},
  booktitle={Proceedings of the IEEE/CVF International Conference on Computer Vision},
  pages={19513--19524},
  year={2025}
}

@article{deng2024causal,
  title={Causal diffusion transformers for generative modeling},
  author={Deng, Chaorui and Zhu, Deyao and Li, Kunchang and Guang, Shi and Fan, Haoqi},
  journal={arXiv preprint arXiv:2412.12095},
  year={2024}
}

@inproceedings{yao2025reconstruction,
  title={Reconstruction vs. generation: Taming optimization dilemma in latent diffusion models},
  author={Yao, Jingfeng and Yang, Bin and Wang, Xinggang},
  booktitle={Proceedings of the IEEE/CVF Conference on Computer Vision and Pattern Recognition},
  pages={15703--15712},
  year={2025}
}

@article{chen2024deep,
  title={Deep compression autoencoder for efficient high-resolution diffusion models},
  author={Chen, Junyu and Cai, Han and Chen, Junsong and Xie, Enze and Yang, Shang and Tang, Haotian and Li, Muyang and Lu, Yao and Han, Song},
  journal={arXiv preprint arXiv:2410.10733},
  year={2024}
}

@article{cheng2025playing,
  title={{Playing with transformer at 30+ FPS via next-frame diffusion}},
  author={Cheng, Xinle and He, Tianyu and Xu, Jiayi and Guo, Junliang and He, Di and Bian, Jiang},
  journal={arXiv preprint arXiv:2506.01380},
  year={2025}
}

@article{huang2025self,
  title={Self forcing: Bridging the train-test gap in autoregressive video diffusion},
  author={Huang, Xun and Li, Zhengqi and He, Guande and Zhou, Mingyuan and Shechtman, Eli},
  journal={arXiv preprint arXiv:2506.08009},
  year={2025}
}

@inproceedings{tian2024u,
  title={{U-DiTs: Downsample tokens in U-shaped diffusion transformers}},
  author={Tian, Yuchuan and Tu, Zhijun and Chen, Hanting and Hu, Jie and Xu, Chao and Wang, Yunhe},
  volume={37},
  pages={51994--52013},
  year={2024},
  booktitle={Advances in Neural Information Processing Systems}
}

@article{sclocchi2025phase,
  title={A phase transition in diffusion models reveals the hierarchical nature of data},
  author={Sclocchi, Antonio and Favero, Alessandro and Wyart, Matthieu},
  journal={Proceedings of the National Academy of Sciences},
  volume={122},
  number={1},
  pages={e2408799121},
  year={2025},
  publisher={National Academy of Sciences}
}

@inproceedings{chen2025towards,
  title={Towards stabilized and efficient diffusion transformers through long-skip-connections with spectral constraints},
  author={Chen, Guanjie and Zhao, Xinyu and Zhou, Yucheng and Qu, Xiaoye and Chen, Tianlong and Cheng, Yu},
  booktitle={Proceedings of the IEEE/CVF International Conference on Computer Vision},
  pages={17708--17718},
  year={2025}
}

@article{russakovsky2015imagenet,
  title={{ImageNet large scale visual recognition challenge}},
  author={Russakovsky, Olga and Deng, Jia and Su, Hao and Krause, Jonathan and Satheesh, Sanjeev and Ma, Sean and Huang, Zhiheng and Karpathy, Andrej and Khosla, Aditya and Bernstein, Michael and Berg, Alexander C. and Li, Fei-Fei},
  journal={International Journal of Computer Vision},
  volume={115},
  number={3},
  pages={211--252},
  year={2015},
  publisher={Springer}
}

@inproceedings{salimans2016improved,
  title={{Improved techniques for training GANs}},
  author={Salimans, Tim and Goodfellow, Ian and Zaremba, Wojciech and Cheung, Vicki and Radford, Alec and Chen, Xi},
  volume={29},
  year={2016},
  booktitle={Advances in Neural Information Processing Systems},
  pages={2234--2242}
}

@inproceedings{heusel2017gans,
  title={{GANs trained by a two time-scale update rule converge to a local Nash equilibrium}},
  author={Heusel, Martin and Ramsauer, Hubert and Unterthiner, Thomas and Nessler, Bernhard and Hochreiter, Sepp},
  volume={30},
  year={2017},
  booktitle={Advances in Neural Information Processing Systems},
  pages={6627--6638}
}

@inproceedings{nash2021generating,
  title={Generating images with sparse representations},
  author={Nash, Charlie and Menick, Jacob and Dieleman, Sander and Battaglia, Peter},
  booktitle={Proceedings of the 38th International Conference on Machine Learning},
  pages={7958--7968},
  year={2021}
}

@inproceedings{kynkaanniemi2019improved,
  title={Improved precision and recall metric for assessing generative models},
  author={Kynk{\"a}{\"a}nniemi, Tuomas and Karras, Tero and Laine, Samuli and Lehtinen, Jaakko and Aila, Timo},
  volume={32},
  year={2019},
  booktitle={Advances in Neural Information Processing Systems},
  pages={3904--3913}
}

@article{ho2022classifier,
  title={Classifier-free diffusion guidance},
  author={Ho, Jonathan and Salimans, Tim},
  journal={arXiv preprint arXiv:2207.12598},
  year={2022}
}

@inproceedings{yin2024improved,
  title={Improved distribution matching distillation for fast image synthesis},
  author={Yin, Tianwei and Gharbi, Micha{\"e}l and Park, Taesung and Zhang, Richard and Shechtman, Eli and Durand, Fredo and Freeman, William T},
  volume={37},
  pages={47455--47487},
  year={2024},
  booktitle={Advances in Neural Information Processing Systems}
}

@inproceedings{yin2024one,
  title={One-step diffusion with distribution matching distillation},
  author={Yin, Tianwei and Gharbi, Micha{\"e}l and Zhang, Richard and Shechtman, Eli and Durand, Fr{\'e}do and Freeman, William T. and Park, Taesung},
  booktitle={Proceedings of the IEEE/CVF Conference on Computer Vision and Pattern Recognition},
  pages={6613--6623},
  year={2024}
}

@inproceedings{sohl2015deep,
  title={Deep unsupervised learning using nonequilibrium thermodynamics},
  author={Sohl-Dickstein, Jascha and Weiss, Eric and Maheswaranathan, Niru and Ganguli, Surya},
  booktitle={Proceedings of the 32nd International Conference on Machine Learning},
  pages={2256--2265},
  year={2015}
}

@inproceedings{nichol2021improved,
  title={Improved denoising diffusion probabilistic models},
  author={Nichol, Alexander Quinn and Dhariwal, Prafulla},
  booktitle={Proceedings of the 38th International Conference on Machine Learning},
  pages={8162--8171},
  year={2021}
}

@inproceedings{kingma2021variational,
  title={Variational diffusion models},
  author={Kingma, Diederik and Salimans, Tim and Poole, Ben and Ho, Jonathan},
  volume={34},
  pages={21696--21707},
  year={2021},
  booktitle={Advances in Neural Information Processing Systems}
}

@inproceedings{karras2022elucidating,
  title={Elucidating the design space of diffusion-based generative models},
  author={Karras, Tero and Aittala, Miika and Aila, Timo and Laine, Samuli},
  volume={35},
  pages={26565--26577},
  year={2022},
  booktitle={Advances in Neural Information Processing Systems}
}

@article{song2020denoising,
  title={Denoising diffusion implicit models},
  author={Song, Jiaming and Meng, Chenlin and Ermon, Stefano},
  journal={arXiv preprint arXiv:2010.02502},
  year={2020}
}

@inproceedings{rombach2022high,
  title={High-resolution image synthesis with latent diffusion models},
  author={Rombach, Robin and Blattmann, Andreas and Lorenz, Dominik and Esser, Patrick and Ommer, Bj{\"o}rn},
  booktitle={Proceedings of the IEEE/CVF Conference on Computer Vision and Pattern Recognition},
  pages={10684--10695},
  year={2022}
}

@inproceedings{ci2025describe,
  title={{Describe, don't dictate: Semantic image editing with natural language intent}},
  author={Ci, En and Guan, Shanyan and Ge, Yanhao and Zhang, Yilin and Li, Wei and Zhang, Zhenyu and Yang, Jian and Tai, Ying},
  booktitle={Proceedings of the IEEE/CVF International Conference on Computer Vision},
  pages={19185--19194},
  year={2025}
}

@article{wan2025wan,
  title={Wan: Open and advanced large-scale video generative models},
  author={{Team Wan} and Wang, Ang and Ai, Baole and Wen, Bin and Mao, Chaojie and Xie, Chen-Wei and Chen, Di and Yu, Feiwu and Zhao, Haiming and Yang, Jianxiao and Zeng, Jianyuan and Wang, Jiayu and Zhang, Jingfeng and Zhou, Jingren and Wang, Jinkai and Chen, Jixuan and Zhu, Kai and Zhao, Kang and Yan, Keyu and Huang, Lianghua and Feng, Mengyang and Zhang, Ningyi and Li, Pandeng and Wu, Pingyu and Chu, Ruihang and Feng, Ruili and Zhang, Shiwei and Sun, Siyang and Fang, Tao and Wang, Tianxing and Gui, Tianyi and Weng, Tingyu and Shen, Tong and Lin, Wei and Wang, Wei and Wang, Wei and Zhou, Wenmeng and Wang, Wente and Shen, Wenting and Yu, Wenyuan and Shi, Xianzhong and Huang, Xiaoming and Xu, Xin and Kou, Yan and Lv, Yangyu and Li, Yifei and Liu, Yijing and Wang, Yiming and Zhang, Yingya and Huang, Yitong and Li, Yong and Wu, You and Liu, Yu and Pan, Yulin and Zheng, Yun and Hong, Yuntao and Shi, Yupeng and Feng, Yutong and Jiang, Zeyinzi and Han, Zhen and Wu, Zhi-Fan and Liu, Ziyu},
  journal={arXiv preprint arXiv:2503.20314},
  year={2025}
}

@article{labs2025flux,
  title={{FLUX.1 Kontext: Flow matching for in-context image generation and editing in latent space}},
  author={Labs, Black Forest and Batifol, Stephen and Blattmann, Andreas and Boesel, Frederic and Consul, Saksham and Diagne, Cyril and Dockhorn, Tim and English, Jack and English, Zion and Esser, Patrick and others},
  journal={arXiv preprint arXiv:2506.15742},
  year={2025}
}

@article{yang2024cogvideox,
  title={{CogVideoX: Text-to-video diffusion models with an expert transformer}},
  author={Yang, Zhuoyi and Teng, Jiayan and Zheng, Wendi and Ding, Ming and Huang, Shiyu and Xu, Jiazheng and Yang, Yuanming and Hong, Wenyi and Zhang, Xiaohan and Feng, Guanyu and others},
  journal={arXiv preprint arXiv:2408.06072},
  year={2024}
}

@article{zhang2025test,
  title={Test-time training done right},
  author={Zhang, Tianyuan and Bi, Sai and Hong, Yicong and Zhang, Kai and Luan, Fujun and Yang, Songlin and Sunkavalli, Kalyan and Freeman, William T and Tan, Hao},
  journal={arXiv preprint arXiv:2505.23884},
  year={2025}
}

@inproceedings{ma2024learningtocache,
  title={{Learning-to-Cache: Accelerating diffusion transformer via layer caching}},
  author={Ma, Xinyin and Fang, Gongfan and Bi Mi, Michael and Wang, Xinchao},
  booktitle={Advances in Neural Information Processing Systems},
  volume={37},
  pages={133282--133304},
  year={2024}
}

@article{chen2024deltadit,
  title={{$\Delta$-DiT: A training-free acceleration method tailored for diffusion transformers}},
  author={Chen, Pengtao and Shen, Mingzhu and Ye, Peng and Cao, Jianjian and Tu, Chongjun and Bouganis, Christos-Savvas and Zhao, Yiren and Chen, Tao},
  journal={arXiv preprint arXiv:2406.01125},
  year={2024}
}

\newpage
\appendix
\noindent In the Appendix, we provide supplementary materials for our work ``Rethinking Cross-Layer Information Routing in
Diffusion Transformers'', organized to follow the main paper.

\section{More Discussion on Related Work}
\label{app:relatedworks}

\subsection{Diffusion Models}

Diffusion models were originally formulated as a finite-step Markov chain that progressively corrupts data with Gaussian noise and learns to reverse the process via a variational bound \citep{sohl2015deep,ho2020denoising}, and \citet{song2020score} unified this view with score matching by recasting the forward and reverse processes as a continuous-time SDE with an equivalent probability-flow ODE. Subsequent work refined the noise schedule and parameterization \citep{nichol2021improved,kingma2021variational,karras2022elucidating, song2020denoising}. To avoid pixel-space cost, latent diffusion performs denoising in the compressed latent space of a pretrained autoencoder \citep{rombach2022high}. A parallel line reformulated generation as learning a deterministic velocity field that transports noise to data, with Flow Matching \citep{lipman2022flow} regressing onto the conditional velocity along a prescribed probability path and Rectified Flow \citep{liu2022flow} favoring straight transport for few-step inference; SiT \citep{ma2024sit} casts diffusion- and flow-based objectives under a single interpolant framework, and Stable Diffusion 3 \citep{esser2024scaling} scales rectified-flow training to large text-to-image models. Diffusion models have since become a dominant generative paradigm, achieving remarkable success in image generation~\citep{wu2025qwen, xie2024sana, flux2024}, image editing~\citep{labs2025flux, ci2025describe, zhang2025test}, video generation~\citep{yang2024cogvideox, wan2025wan}, and related visual synthesis tasks.

\section{Proof of Proposition and Empirical Verification}
\label{app:chunk_proof}

For convenience, we restate the proposition.
\setcounter{proposition}{0}
\renewcommand{\theHproposition}{appendix.\arabic{proposition}}
\begin{proposition}[U-shaped cost of chunked aggregation]
Let $L>0$ and $\alpha\in(0,1)$. Then $\mathcal{L}(S)$ in
Eq.~\eqref{eq:chunk_cost} is strictly decreasing on $(0,S^\star)$ and strictly
increasing on $(S^\star,\infty)$, where
\begin{equation}
S^\star=\sqrt{L \cdot \frac{1-\alpha}{1+\alpha}}.
\end{equation}
Consequently, $\mathcal{L}(S)$ is U-shaped and has a unique global minimizer at
$S^\star$.
\end{proposition}

\begin{proof}
Differentiating Eq.~\eqref{eq:chunk_cost} yields
\[
\mathcal{L}'(S)
=
\frac{S^2 - L}{S(S^2 + L)}
+
\frac{\alpha}{S}
=
\frac{(1+\alpha)S^2-(1-\alpha)L}{S(S^2+L)} \ .
\]
Since $S(S^2+L)>0$ for all $S>0$, the sign of $\mathcal{L}'(S)$ is determined by
\[
(1+\alpha)S^2-(1-\alpha)L \ .
\]
Thus,
\[
\mathcal{L}'(S)<0 \quad \text{for} \quad S<S^\star \ ,
\qquad
\mathcal{L}'(S)=0 \quad \text{for} \quad S=S^\star \ ,
\qquad
\mathcal{L}'(S)>0 \quad \text{for} \quad S>S^\star \ ,
\]
where
\[
S^\star=\sqrt{L \cdot \frac{1-\alpha}{1+\alpha}} \ .
\]
Therefore, $\mathcal{L}(S)$ is strictly decreasing on $(0,S^\star)$ and strictly
increasing on $(S^\star,\infty)$, which proves that $S^\star$ is the unique
global minimizer.
\end{proof}

\paragraph{Empirical agreement.}
For SiT-XL/2 (depth $28$, two sublayers per chunk, hence $L = 56$), Eq.~\eqref{eq:S_star} predicts $S^\star \in [3.7, 4.9]$ over the realistic range $\alpha \in [0.4, 0.6]$, identifying $S = 4$ as the model-predicted optimum. This agrees quantitatively with Tab.~\ref{tab:chunk_size}: $S = 4$ improves on both the under-compressed extreme $S = 1$ (no chunking, $|\mathcal{S}| = L$, dominated by the routing-entropy term) and the over-compressed $S = 8$ (small $|\mathcal{S}|$ but large per-summary distortion). We therefore use $S = 4$ throughout the main experiments. Since Eq.~\eqref{eq:S_star} predicts $S^\star$ to scale as $\sqrt{L}$, we further conjecture that scaling DiTs to substantially deeper backbones will require a proportionally larger chunk size.

\section{Additional Implementation Details}
\label{impl}

\subsection{Experimental Configuration}

Unless otherwise specified, all ImageNet experiments follow the training recipe of SiT~\citep{ma2024sit}. We train models with a global batch size of $1024$, a learning rate of $1 \times 10^{-4}$, and bfloat16 mixed precision. All compared models are trained under the same optimization and data-processing settings unless explicitly stated otherwise. For experiments involving REPA~\citep{yu2025repa}, we adopt the original REPA configuration. Specifically, we use DINOv2-B as the pretrained visual encoder, set the representation-alignment coefficient to $0.5$, and apply the alignment loss to the hidden representation at the eighth layer. No additional modifications are introduced beyond replacing the residual routing mechanism with
the proposed \textsc{DAR}. To ensure a fair comparison, we rerun the SiT and REPA baselines under our experimental setting rather than using previously reported results.

\subsection{Compute Resources}
Experiments were conducted on NVIDIA H20 GPUs with 192 CPU cores and 1024 GB of system memory. Training throughput and memory measurements with the fused implementation are reported in Appendix~\ref{app:infrastructure}.

\subsection{Additional Implementation Details for DAR}
\label{app:final_aggregator}

\paragraph{Dedicated final aggregator.} In our design, the final aggregator that produces the input to the prediction layer has access to the input embedding, the prior chunk summaries, and all raw sublayer outputs of the last chunk
\begin{equation}
\mathcal{S}_{\text{final}} \;=\; \underbrace{\{c_0, c_1, \ldots, c_{N-1}\}}_{\text{embedding and prior summaries}} \;\cup\;
\underbrace{\{v_{(N-1)S+1}, v_{(N-1)S+2}, \ldots, v_L\}}_{\text{raw last-chunk sublayer outputs}} \ .                                    
\end{equation}
This is in contrast to AttnRes, whose final layer aggregates strictly over the $N$ chunk summaries. The intuition is that the most recent sublayer outputs carry the most task-specific signal; thus, exposing them in raw form, rather than compressed into a single summary, lets the final layer recover fine-grained information that the chunk-level summary would otherwise discard. This yields about a 2-point FID gain after $200\text{K}$ training iterations.

\paragraph{REPA-specific implementation.} For experiments that combine \textsc{DAR} with REPA, we depart from the dedicated final aggregator described above and instead let the last chunk reuse its own MLP aggregator parameters --- the static/dynamic query ($w_l$ or $W_q^{(l)}$) and the per-source $\mathrm{RMSNorm}$ --- to perform the final aggregation. This is an empirically motivated design that improves performance.

\section{Direct Comparison with Attention Residuals}
\label{app:attnres_comparison}

We compare a direct application of Attention Residuals~\citep{team2026attention} to SiT-XL/2 with \textsc{DAR} under the same training and evaluation protocol on ImageNet $256\times256$. Table~\ref{tab:attnres_direct} reports FID at the available checkpoints.

\begin{table}[htbp]
\centering
\small
\setlength{\tabcolsep}{5pt}
\renewcommand{\arraystretch}{1.15}
\begin{tabular}{lrrrrrrrr}
\toprule
\textbf{Method} & \textbf{100K} & \textbf{200K} & \textbf{300K} & \textbf{400K} & \textbf{500K} & \textbf{600K} & \textbf{700K} & \textbf{800K} \\
\midrule
Attention Residuals & 21.78 & 14.28 & 11.52 & 10.28 & 9.51 & 8.99 & 8.71 & 8.73 \\
\textsc{DAR} (Static c4) & 17.39 & 10.12 & 8.62 & 7.97 & 7.81 & 7.56 & -- & -- \\
\bottomrule
\end{tabular}
\caption{Direct comparison with Attention Residuals on SiT-XL/2, ImageNet $256\times256$. We report FID$\downarrow$ using ODE sampling without classifier-free guidance. \textsc{DAR} uses the static query with explicit timestep injection and chunk size $S=4$. Dashes indicate unreported checkpoints.}
\label{tab:attnres_direct}
\end{table}

Attention Residuals reaches its best reported FID of $8.71$ at $700\text{K}$ iterations and yields $8.73$ at $800\text{K}$. \textsc{DAR} reaches $7.56$ at $600\text{K}$, improving on the best reported Attention Residuals result by $1.15$ FID with fewer training iterations. This comparison supports the additional benefit of \textsc{DAR}'s diffusion-specific routing design over directly applying Attention Residuals to SiT.

\section{Details of Large-Scale T2I Model Post-Training}
\label{t2i}

For large-scale T2I post-training, we apply DMD to Qwen-Image~\citep{wu2025qwen} with \textsc{DAR} inserted into the MM-DiT backbone. We use LoRA fine-tuning with rank $64$, and train the student branch with a learning rate of $5 \times 10^{-6}$ together with a fake branch learning rate of $2 \times 10^{-6}$. Distillation is performed with $4$ denoising steps and guidance scale $4.0$. We train at $1024^2$ resolution using bfloat16 mixed precision and a per-GPU batch size of $1$. More visualizations are provided in the supplementary materials.

\paragraph{Quantitative evaluation.} We compare standard routing and \textsc{DAR} under matched Qwen-Image DMD settings using the same GenEval2 prompts. Both metrics in Table~\ref{tab:dmd_quantitative} are computed relative to the teacher outputs: LPIPS measures perceptual discrepancy, while the high-frequency-ratio (HFR) deviation derived from RAPSD compares frequency content above the threshold $0.2$. Lower values indicate closer agreement with the teacher.

\begin{table}[htbp]
\centering
\small
\setlength{\tabcolsep}{12pt}
\renewcommand{\arraystretch}{1.15}
\begin{tabular}{lcc}
\toprule
\textbf{Method} & \textbf{LPIPS}$\downarrow$ & \textbf{RAPSD HFR($>0.2$) deviation}$\downarrow$ \\
\midrule
Standard routing + DMD & 0.538 & $+0.340$ \\
\textsc{DAR} + DMD & \textbf{0.512} & $\mathbf{+0.215}$ \\
\bottomrule
\end{tabular}
\caption{Quantitative comparison on Qwen-Image DMD using GenEval2 prompts. Both metrics are measured against the teacher outputs; lower is better.}
\label{tab:dmd_quantitative}
\end{table}

\textsc{DAR} reduces teacher-relative LPIPS from $0.538$ to $0.512$ and the high-frequency-ratio deviation from $+0.340$ to $+0.215$. These results show that \textsc{DAR} stays closer to the teacher in both perceptual similarity and frequency content, supporting the qualitative detail-preservation observations.

\FloatBarrier
\section{Infrastructure}
\label{app:infrastructure}

\begin{figure*}[t]
\centering
\includegraphics[width=1\textwidth]{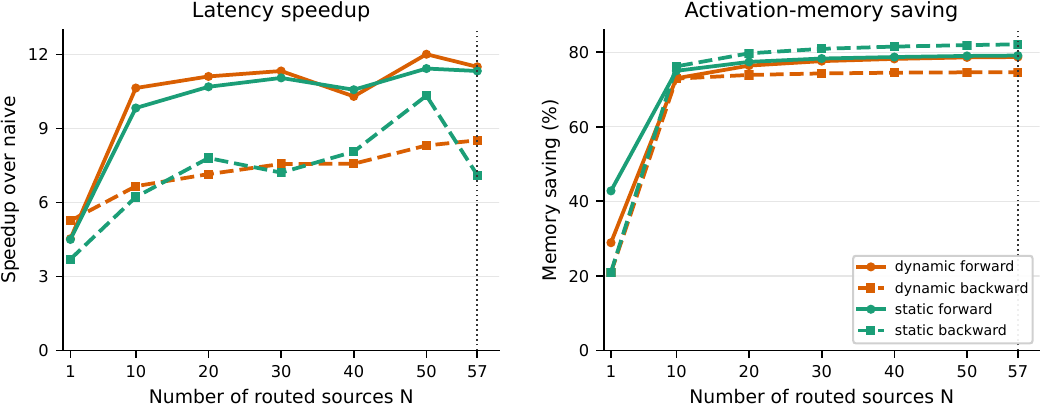}
\caption{Infrastructure benchmark of the fused Triton implementation for \textsc{DAR}. Left: latency speedup over a naive implementation as a function of the number of routed sources $N$. Right: activation-memory saving, shown for the dynamic/static variants and the forward/backward passes.}
\label{fig:infra_benchmark}
\end{figure*}

A naive implementation of \textsc{DAR}'s vertical aggregation $h_l = \sum_{i<l} \alpha_{i\to l}(t)\,v_i$ decomposes into separate kernels for per-source RMSNorm, query--key dot product, softmax, and weighted sum, each launching its own CUDA kernel, materializing $[N, B, T, D]$-shaped intermediates in HBM, and reading the source tensor four times per forward pass; since $N$ scales with depth, this baseline is both memory- and bandwidth-bound. We collapse the entire forward path into a single Triton kernel that uses an online-softmax recurrence to fuse the normalization constant with the weighted accumulator in one streaming loop over the $N$ sources, so that $\{v_i\}_{i<l}$ is read from HBM exactly once and all per-source intermediates --- $\mathrm{RMS}(v_i)$, $k_i$, $q^{\top}k_i$, $\exp(\cdot)$ --- live entirely in registers. The backward kernel applies the same idea in two passes: it first streams the sources to recover the softmax statistics $(m, Z, s)$, then streams them again to recompute the RMSNorm intermediates on the fly and emit $\partial\mathcal{L}/\partial v_i$, $\partial\mathcal{L}/\partial q$, and $\partial\mathcal{L}/\partial w_{\textsc{norm}}$ in two HBM reads instead of four to five; we further fuse the downstream LayerNorm and adaLN \texttt{modulate} into the same kernel. Microbenchmarks at the SiT-XL/2 working point ($N=57$) in Fig.~\ref{fig:infra_benchmark} show the fused kernel reducing forward latency from $22.5$\,ms to $1.96$\,ms ($11.5\times$) and backward from $115.8$\,ms to $13.6$\,ms ($8.5\times$) for the dynamic variant, with peak activation memory dropping by $78.7\%$ in the forward and $74.6\%$ in the backward pass (and up to $82.1\%$ for the static variant); the savings grow monotonically with $N$, keeping the chunked aggregator viable as DiTs scale to deeper backbones. The fused kernel is numerically equivalent to the reference PyTorch path up to floating-point reordering and serves as a drop-in replacement used throughout the main paper.

\paragraph{Training-time efficiency.} Table~\ref{tab:training_efficiency} compares SiT-XL/2 and \textsc{DAR} static c4 on the same NVIDIA H20 setup with a global batch size of $1024$. Both models use the latest fused implementation; fused operators applicable to SiT are also enabled for the baseline. The measured training rates are $1.83$ and $1.73$ steps/s, respectively, corresponding to a $5.8\%$ per-step overhead for \textsc{DAR}.

\begin{table}[htbp]
\centering
\small
\setlength{\tabcolsep}{7pt}
\renewcommand{\arraystretch}{1.15}
\begin{tabular}{lrrrr}
\toprule
\textbf{Model} & \textbf{Time (s/step)} & \textbf{Steps/s} & \textbf{Images/s} & \textbf{Peak memory (GB)} \\
\midrule
SiT-XL/2 & 0.55 & 1.83 & 1874 & 54.56 \\
\textsc{DAR} (Static c4) & 0.58 & 1.73 & 1771 & 69.97 \\
\bottomrule
\end{tabular}
\caption{Training efficiency with fused operators on SiT-XL/2, ImageNet $256\times256$. Step times are rounded to two decimal places.}
\label{tab:training_efficiency}
\end{table}
\FloatBarrier

Using the reported $8.75\times$ iteration reduction to the baseline-quality target (FID $9.67$), the measured training rates imply an estimated wall-clock speedup of
\[
\frac{T_{\mathrm{SiT}}}{T_{\textsc{DAR}}}
= \frac{N_{\mathrm{SiT}}}{N_{\textsc{DAR}}}\,\frac{r_{\textsc{DAR}}}{r_{\mathrm{SiT}}}
\approx 8.75\times\frac{1.73}{1.83}
\approx 8.27\times,
\]
where $N$ is the number of training iterations and $r$ is the measured rate in steps/s. Thus, after accounting for the per-step overhead, the estimated training time is approximately $12.1\%$ of the baseline's at this target FID.

\section{Limitations and Future Work}
\label{limitations}

We view the most compelling next step as pushing \textsc{DAR} along the two scale axes that dominate modern generative Transformers: large-scale pretraining and large-scale post-training. On the pretraining side, state-of-the-art T2I and T2V backbones such as MM-DiT~\citep{esser2024scaling}, Qwen-Image~\citep{wu2025qwen}, FLUX~\citep{flux2024}, and HunyuanVideo~\citep{kong2024hunyuanvideo} routinely scale to several billion parameters with substantially deeper Transformer stacks, where the PreNorm-dilution symptoms diagnosed in Section~\ref{sec:diagnose} should manifest more severely than in SiT-XL/2; Eq.~\eqref{eq:S_star} further predicts that the optimal chunk size grows as $\sqrt{L}$, hinting that the headroom for \textsc{DAR} may widen rather than saturate with depth, which makes a systematic scaling study on multi-billion-parameter MM-DiT and video-DiT pretraining the most natural and informative follow-up direction we envision. On the post-training side, our preliminary DMD~\citep{yin2024one,yin2024improved} experiment on Qwen-Image equipped with \textsc{DAR} (Appendix~\ref{t2i}) suggests that the better-conditioned gradient flow afforded by adaptive routing has a detail-preserving effect on otherwise brittle distillation procedures where the vanilla counterpart loses fine details, and we plan to extend this observation to a broader family of post-training objectives---including supervised fine-tuning, RL-style preference optimization, and few-step distillation---across multiple large-scale T2I and T2V backbones, to assess whether \textsc{DAR} can serve as a general-purpose technique across diverse diffusion post-training settings.



\end{document}